\documentclass[times, preprint, 12pt, nopreprintline]{elsarticle}

\usepackage{amssymb}
\usepackage{amsmath, amsfonts, amsthm}
\usepackage{algorithm}
\usepackage{algorithmic}
\usepackage{booktabs}
\usepackage{wrapfig}

\DeclareMathOperator{\Tr}{Tr}

\newtheorem{theorem}{Theorem}
\newtheorem{definition}{Definition}

\begin{document}
\begin{frontmatter}



\title{Deep Matrix Factorization with Adaptive Weights for Multi-View Clustering}


\author[alteca]{Yasser KHALAFAOUI\corref{cor1}}
\cortext[cor1]{Corresponding author}

\affiliation[alteca]{organization={ALTECA},
            addressline={\{mykhalafaoui, mlovisetto\}@alteca.fr}, 
            city={Lyon},
            country={France}}

\author[lipn]{Basarab MATEI} 
\author[alteca]{Martino LOVISETTO} 
\affiliation[lipn]{organization={Sorbonne Paris Nord University},
            addressline={basarab.matei@lipn.univ-paris13.fr}, 
            city={Villetaneuse},
            country={France}}

\author[cy]{Nistor GROZAVU} 

\affiliation[cy]{organization={CY Cergy Paris University},
            addressline={nistor.grozavu@cyu.fr}, 
            city={Pontoise},
            country={France}}

\begin{abstract}
Recently, deep matrix factorization has been established as a powerful model for unsupervised tasks, achieving promising results, especially for multi-view clustering. However, existing methods often lack effective feature selection mechanisms and rely on empirical hyperparameter selection. To address these issues, we introduce a novel Deep Matrix Factorization with Adaptive Weights for Multi-View Clustering (DMFAW). Our method simultaneously incorporates feature selection and generates local partitions, enhancing clustering results. Notably, the features weights are controlled and adjusted by a parameter that is dynamically updated using Control Theory inspired mechanism, which not only improves the model's stability and adaptability to diverse datasets but also accelerates convergence. A late fusion approach is then proposed to align the weighted local partitions with the consensus partition. Finally, the optimization problem is solved via an alternating optimization algorithm with theoretically guaranteed convergence. Extensive experiments on benchmark datasets highlight that DMFAW outperforms state-of-the-art methods in terms of clustering performance.
\end{abstract}

\begin{keyword}
Multi-view Clustering, Matrix Factorization, Unsupervised Learning


\end{keyword}

\end{frontmatter}


\section{Introduction}
In the era of big data, one frequently encounters datasets with multiple sources or views, each offering a unique perspective on the underlying phenomena. These diverse views may capture distinct aspects, including textual information, visual features, or temporal dynamics. To exploit intrinsic information across these different views, substantial research efforts  have been dedicated to the development and enhancement of multi-view clustering (MVC) \cite{kumar2011co,li2016multiple,wang2019multi,tan2020unsupervised,wang2020learning,chen2021relaxed}, with a particular emphasis on Matrix Factorization-based approaches \cite{zhao2017multi,wen2018incomplete}.  Notably, Non-negative Matrix Factorization (NMF) methods have demonstrated their effectiveness in handling high-dimensional data while capturing underlying structures across different views \cite{liu2013multi}.\par
NMF has been applied in a range of fields such as clustering \cite{ding2005equivalence}, document understanding \cite{wang2010weighted} and representation learning \cite{trigeorgis2014deep}.
The core idea behind NMF is to decompose a given high-dimensional data matrix into two low rank matrices. Notably, NMF imposes a non-negativity constraint during the factorization process. This constraint simplifies the interpretation of the resulting matrices, allowing for a more intuitive and interpretable analysis \cite{brunet2004metagenes}. By extending NMF to accommodate diverse data views, one is able to integrate and exploit the complementary information from these different views. This extension enhances the accuracy, robustness and interpretability of the clustering process.
In the literature, many NMF-based multi-view clustering methods have been proposed. Some approaches \cite{chang2014multi,jiang2014semi,liu2014partially} introduce a sparse model that learns discrete clustering labels based on the shared latent representation. Others \cite{liu2018consensus,khalafaoui2023joint} propose a joint multi-view consensus clustering method to address late fusion (i.e., partition level) and the mutual update between the consensus partition matrix and the local partition matrices. However, most single layer NMF methods are unable to extract deeper and hidden information of data which may impact the clustering results. Recently, a number of deep NMF-based multi-view clustering methods have been developed. In order to guide the shared representation learning in each view, \citet{zhao2017multi} combine a deep semi-NMF structure to extract hidden information with a graph regularizer. \citet{huang2020auto} suggest utilizing a collaborative deep matrix decomposition framework to learn the hidden representations. To extract multi-view information, \citet{zhang2021multi} fused each view's partition representations, found by deep matrix decomposition, into a consensus partition representation.\par
While deep NMF-based multi-view clustering approaches have shown promising results, significant challenges persist. A major issue is that these methods typically perform clustering across the entire feature space, without distinguishing between more and less important features, which can lead to suboptimal clustering results. Additionally, the effectiveness of these approaches is often hampered by the need for precise selection of various hyperparameters, such as the number of layers, the dimensions of each layer, and specially parameters that control the degree of feature selection (i.e., strong or weak feature selection). In the literature, the latter parameters are usually determined analytically and are not directly tied to the system’s performance. This lack of a performance-driven mechanism for feature selection means that many approaches fail to accurately capture the most relevant features, ultimately limiting the model's ability to produce high-quality clustering results.\par
In order to address these issues, this paper proposes Deep Matrix Factorization with Adaptive Weights for Multi-view Clustering (DMFAW). The proposed method emphasizes the importance of feature selection for improving clustering results and employs a weighted Deep Semi-NMF to simultaneously generate local partition matrices and select important features. Additionally, the parameter controlling the degree of feature selection is updated dynamically via a method inspired by PI Stepsize Control approach from the Control Theory field \cite{gustafsson1988pi}. Finally, a late fusion approach is applied to obtain a consensus partition matrix from the local partition matrices. Our contributions can be summarized as follows:
\begin{itemize}
    \item We propose DMFAW. A weighted Deep Semi-NMF approach is used for simultaneous generation of local partitions and feature selection, enhancing the multi-view clustering performance.
    \item We introduce a dynamic feature selection parameter update mechanism inspired by Control Theory's PI Stepsize Control, enhancing model stability and adaptability to diverse datasets while accelerating convergence.
    \item We conduct extensive experiments on real-world datasets, validating the effectiveness and efficiency of DMFAW. The results demonstrate better performance compared to other state-of-the-art methods.
\end{itemize}

\section{Related Work}
\subsection{Multi-view Clustering} 
It aims to get a high-quality clustering result by utilizing heterogeneous information from different views. \citet{kumar2011co} propose a Co-training Approach for Multi-View Spectral Clustering, which combines semi-supervised learning and spectral clustering for multi-view data analysis. This approach alternates between self-training, in which the local clusterings mutually update the other views, and label propagation where the updated views are used to re-label the data points which in turn are used to refine the clustering results. Multi-View clustering via Late Fusion Alignment Maximization (MVC-LFA) \cite{wang2019multi} is a framework that uses late fusion to integrate multiple views. It jointly and simultaneously optimizes the consensus representation, transformation matrices and the weight coefficients via maximizing the alignment between consensus and weighted local representations. \citet{chen2021relaxed} propose Multi-View Clustering in Latent Embedding Space (MCLES) which jointly learns a comprehensive latent embedding representation matrix, an accurate cluster indicator matrix and a robust global similarity matrix in a unified framework by seamlessly leveraging the interaction between these matrices. While these methods demonstrate overall good clustering performance and runtime efficiency, they exhibit limitations in capturing hidden, hierarchical relationships within the data, which can be crucial for learning better latent data representations.

\subsection{Deep Matrix Factorization} 
In many instances, the datasets we encounter encompass a variety of distinct features. To address this challenge, the concept of Deep Semi-NMF has emerged \cite{trigeorgis2014deep}. In this framework, a data matrix is factorized into $m + 1$ factors, while imposing a non-negativity constraint on the implicit representations. This constraint extends the interpretability of each layer's representation within this hierarchical structure, allowing for a natural clustering interpretation. Unfortunately, theses method can only handle single-view data.
By combining deep matrix factorization with multi-view horizontal collaboration, Multi-view Clustering via Deep Matrix Factorization (DMF-MVC) \cite{zhao2017multi} learns layer-wise latent representations, with each layer leveraging complementary information from previous layers. Furthermore, a constraint is imposed to ensure that multi-view data shares the same representation following multi-layer factorization. To preserve the geometric structure inherent in each data view, the authors introduce a graph Laplacian as a regularization term. However, it's worth noting that the authors empirically determine view weights. Auto-weighted Multi-View Clustering via Deep Matrix Factorization (Aw-DMVC) \cite{huang2020auto}  addresses this critical challenge in multi-view learning by enabling automatic weight assignment to different views. This adaptive method improves the performance of the proposed approach and enhances its flexibility compared to methods that rely on manually assigned weights. However, Aw-DMVC doesn't incorporate ensemble learning. On the other hand, Multi-View Clustering via Deep Matrix Factorization and Partition Alignment (MVC-DMF-PA) \cite{zhang2021multi} integrates representation learning and the late fusion stage within a framework, allowing them to mutually guide each other towards the generation of the consensus representation matrix. To guide the learning process, partition alignment is used to align the latent representations across different views, ensuring that they represent the same underlying clusters. Different from MVC-DMF-PA, our deep matrix factorization embeds feature selection while dynamically updating the parameter controlling its degree using Control Theory's principles.

\section{Methodology}
In this section, we introduce a novel adaptive framework for multi-view clustering, termed Deep Matrix Factorization with Adaptive Weights for Multi-View Clustering (DMFAW). Our approach, illustrated in Fig. \ref{fig:proposed_approach}, enhances upon existing methods by integrating feature selection alongside a cross-domain Control Theory principle to dynamically update weight parameters. First, we show how DMFAW effectively captures important features for multi-view clustering task. Subsequently, we provide an in-depth explanation of our weights parameter update mechanism utilizing PI stepsize control. Finally, we present a multi-view late-fusion strategy and provide theoretical analysis.

\begin{figure}[t]
    \centering
    \includegraphics[width=1\columnwidth]{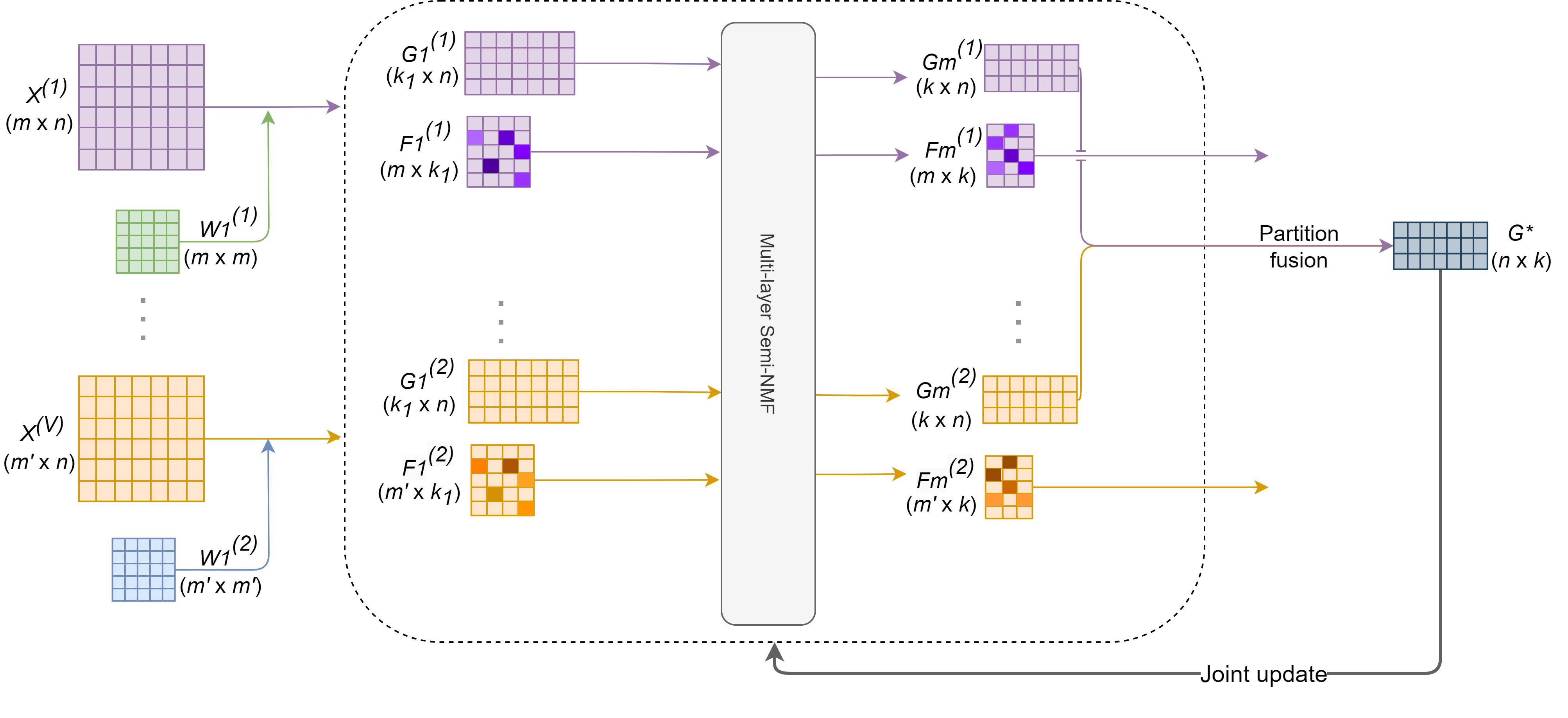}
    \caption{A graphical representation of our proposed solution, DMFAW. The model integrates a weight matrix, denoted as $W^{(v)}$, for feature selection in each view's local clustering phase for each input matrix $X^{(v)}$. Then, a consensus partition matrix is generated based on local partitions $\{G_m^{(v)}\}_{v=1}^V$, permutation matrices $\{M^{(v)}\}_{v=1}^V$ and the average partition region $A$.  Here, $F_i^{(v)}$ and $G_i^{(v)}$ represent the mapping and partition matrices of the $i$-th layer respectively, while $G^*$ represents the consensus partition matrix.}
    \label{fig:proposed_approach}
\end{figure}

\subsection{Weighted Deep Matrix Factorization}
Traditional matrix factorization methods are constrained by their inherent shallowness, limiting their capacity to uncover hierarchical features. They decompose the data matrix $X\in \mathbb{R}^{d\times n}$, comprising $d$ dimensions and $n$ samples, into two factors $F\in \mathbb{R}^{d\times k}$ and $G\in \mathbb{R}^{k\times n}$, representing the mapping and partition matrices, respectively—where $k$ denotes the rank. On the other hand, deep matrix factorization draws inspiration from the successes of deep learning, enabling the extraction of multiple layers of features hierarchically, thus providing novel insights across a wide array of applications \cite{de2021survey}.\newline
Given multi-view data matrices $\{X^{(v)}\}_{v=1}^V$ with $n$ samples, $V$ views and $d_v$ dimensions , deep matrix factorization decomposes each data matrix into $m+1$ factors. Initially, it performs the first factorization $F_1G_1$. Subsequently, in a cascading manner, $G_1$ undergoes further decomposition into $F_2G_2$, and this process iterates until the last partition matrix $G_m$ is obtained. The objective function of multi-view deep matrix factorization is formulated as follows,
\begin{align}
    \min_{F^{(v)}_i, G^{(v)}_i} &\sum_{v=1}^{V} \Vert X^{(v)} - F^{(v)}_1F^{(v)}_2 \cdots F^{(v)}_mG^{(v)}_m \Vert^2_F, \nonumber \\
     & s.t. \hspace{0.1cm} F_i^{(v)} \geq 0, G_i^{(v)} \geq 0 \hspace{0.1cm} i=1,2,\cdots m.
\end{align} 
Additionally, when dealing with unconstrained input data matrices—those that may contain mixed signs—a Semi-NMF approach proves advantageous. In Semi-NMF, only one of the output matrices is constrained to have non-negative values, while the other remains unconstrained \cite{ding2008convex}.\newline
Existing deep matrix factorization approaches tend to treat all data features equally, making them susceptible to the influence of irrelevant or noisy features \cite{wang2010weighted}. To mitigate this, the proposed weighted deep matrix factorization method introduces a feature weighting process to better control feature relevance. It's defined as,
\begin{align}
    \min_{\substack{F_i^{(v)},G_i^{(v)},\\ W^{(v)}}} & \sum_{v=1}^{V} \Vert W^{(v)}(X^{(v)} - F_1^{(v)}F_2^{(v)}\cdots F_m^{(v)}G_m^{(v)})\Vert^2_F, \nonumber \\
    & st. \sum_{d} {(W_{d}^{(v)})}^p = 1, 
\end{align}
where $W^{(v)} \in \mathbb{R}^{{d_v}\times {d_v}}$ is a diagonal matrix indicating the weights of the features of $X^{(v)}$, and $p$, which is introduced in the next section, is a parameter that controls the degree of feature selection. This parameter allows the model to dynamically adjust the influence of different features, enabling performance-driven and effective feature selection, ultimately leading to improved clustering results.

\subsection{Adaptive Feature Selection}
One of the challenging aspects of multi-view clustering is choosing the right parameters, and particularly in our case those controlling feature selection among different views. Existing approaches often rely on empirical or analytical methods to determine these parameters. However, these static approaches may fall short in capturing the dynamic and intricate relationships inherent in multi-view data \cite{wang2010weighted,zhao2017multi,khalafaoui2023joint}. In contrast to these methodologies, we introduce a novel approach inspired by control theory principles to dynamically update the aforementioned parameters.\newline
The integration of control theory techniques, particularly the Proportional-Integral (PI) controller, offers an interesting approach for enhancing adaptability and optimizing model performance \cite{wanner1996solving}. The PI controller is known for its ability to dynamically adjust system parameters based on the integral of past errors and the current error. In the context of machine learning and multi-view clustering, the PI controller becomes an interesting tool to dynamically adjust parameters and steer the model towards optimal solution. \newline
\sloppy At each iteration, the model computes the global loss, representing the disparity between the input data matrix $X^{(v)}$ and the corresponding factorization $F_1^{(v)}F_2^{(v)}\cdots F_m^{(v)}G_m^{(v)}$ and between the consensus $G^*$ and local partition matrices $\{G_m^{(v)}\}_{v=1}^V$. Then, both the past and current losses are evaluated, guiding the adaptive update of the parameter $p$, which controls the degree of feature selection. The iterative process continues, while contributing to the solution convergence and stability.\newline
Following the work related to PI stepsize control \cite{gustafsson1988pi}, which has been shown to enhance the regularity of error estimates, we define our adaptive feature selection parameter term as follows:

\begin{definition}
    Let $p$ be the weight parameter, $closs$ and $ploss$ represent the current and previous computed losses, respectively. $Tol$ is the tolerance for the current loss per iteration. The update rule for $p$ is defined as,
    \begin{equation}\label{eq:update_p}
    p \leftarrow p\cdot \left(\frac{Tol}{|closs|}\right)^{n_1} \left(\frac{|ploss|}{|closs|}\right)^{n_2},
\end{equation}
with $n_1$ and $n_2$ hyperparameters that control the influence of the current loss and the loss update on the weight parameter. $Tol$ is defined as,
\begin{equation}\label{eq:update_tol}
    Tol \leftarrow Tol\cdot \left(1 + \frac{|closs - ploss|}{ploss} \right ).
\end{equation}
\end{definition}

\subsection{Learning the Consensus Partition}
Building upon the methodologies proposed by \cite{zhang2021multi,wang2019multi} for late fusion,  we derive the consensus partition matrix $G^*$ from the local partitions $\{G_m^{(v)}\}_{v=1}^V$ obtained from each individual view. This is achieved by maximizing the alignment between the local partition matrices and the consensus partition matrix through an optimal permutation matrix $M \in \mathbb{R}^{k\times k}$. This permutation matrix unifies the different representations present in each local partition matrix. Additionally, we introduce the matrix $A\in\mathbb{R}^{n\times n}$ which represents the average partition region. The latter helps prevent the consensus partition $G^*$ from deviating from the average partition observed prior to the fusion process.\newline
Eventually, our proposed deep matrix factorization with adaptive weights model can be formulated as,

\begin{align}\label{eq:obj_func}
    \min_{\substack{F_i^{(v)},G_i^{(v)},G^* \\ W^{(v)}, M^{(v)}, \beta^{(v)}}} & \sum_v^V \Vert W^{(v)}(X^{(v)} - F_1^{(v)}F_2^{(v)}\cdots F_m^{(v)}G_m^{(v)}) \Vert^2_F \nonumber \\ 
    & - \lambda \Tr\left (G^*A\sum_v^V\beta^{(v)}G_m^{(v)T}M^{(v)}\right ),   \\
     st.\hspace{0.2cm} G_i^{(v)} \geq 0,& M^{(v)}M^{(v)T} = I_k,  
      \sum_d {(W_d^{(v)})}^p = 1, \beta^{(v)} \geq 0, \nonumber
\end{align}
where $\beta^{(v)}$ is the weighting coefficient of each local partition.

\section{Optimization}
In the following, we derive a six-step alternate optimization algorithm in order to solve Eq. \eqref{eq:obj_func}. Note that, for each view, we need to optimize $F_i^{(v)}$ and $G_i^{(v)}$ layer by layer, i.e., first $F_1^{(v)}$ and $G_1^{(v)}$ until $F_m^{(v)}$ and $G_m^{(v)}$ are updated. Following \cite{ding2013nonnegative,zhang2021multi} we implement a clustering-based initialization, using Semi-NMF, for all the factors $F_i^{(v)}$,$G_i^{(v)}$ in order to mitigate the problem of non-uniqueness of the aforementioned factorization, and expediate the approximation of the variables. 
\paragraph{\textbf{Subproblem of updating $G^*$}}
With $F_i^{(v)}$, $G_i^{(v)}$, $W^{(v)}$, $M^{(v)}$, $\beta^{(v)}$ fixed, the optimization Eq. \eqref{eq:obj_func} can be written as follows,
\begin{equation}\label{eq:opt_cons}
    \min_{G^*} - \Tr(G^*U), st. G^*G^{*T} = I_k,
\end{equation}
where $U = A\sum_v^V\beta^{(v)}G_m^{(v)T}M^{(v)}$. This problem can be solved by taking the singular value decomposition (SVD) of the given matrix $U$. Furthermore, there exists a closed-form solution, which is provided by the following Theorem.
\begin{theorem}\label{th:cf-solution}
    If the matrix $U$, defined previously, has an economic rank-$k$ singular value decomposition form, then the optimization problem in Eq. \ref{eq:opt_cons} has a closed-form solution defined as,
    \begin{equation}\label{eq:closed-form}
        G^* = VS^T,
    \end{equation}
    where $V \in \mathbb{R}^{k\times k}$ and $S \in \mathbb{R}^{n\times k}$ are the right and left singular vectors respectively.
\end{theorem}

\begin{proof}
    The matrix $U$ can be expressed in terms of its singular value decomposition as $U = SDV^T$. We can then rewrite Eq. \eqref{eq:opt_cons} as follows,
    \begin{equation}
        \min_{G^*} - \Tr(G^*SDV^T), st. G^*G^{*T} = I_k.
    \end{equation}
    Since $S$ and $V$ are orthogonal matrices, the optimization problem is equivalent to,
    \begin{equation}
        \min_{G^*} - \Tr(G^*D), st. G^*G^{*T} = I_k.
    \end{equation}
    Utilizing the orthogonality constraint and the properties of orthogonal matrices, a closed-form solution for $G^*$ exists and is defined in Eq. \eqref{eq:closed-form}. This completes the proof.
\end{proof}

\paragraph{\textbf{Subproblem of updating \textbf{$F_i^{(v)}$}}}
With $G_i^{(v)}$, $W^{(v)}$, $M^{(v)}$, $G^*$, $\beta^{(v)}$ fixed, the optimization problem in Eq. \eqref{eq:obj_func} is equivalent to,
\begin{equation}
    \min_{F_i^{(v)}} \mathcal{C} = \min_{F_i^{(v)}} \Vert X^{(v)} -  Z^{(v)}F_i^{(v)}G_i^{(v)} \Vert^2_W,
\end{equation}
where $Z = F_1^{(v)}\cdots F_{i-1}^{(v)}$. Setting $\partial C/\partial F_i^{(v)}=0$, we get the following solution
\begin{equation}\label{eq:update_F}
    F_i^{(v)} = Z^\dag X^{(v)} G_i^{(v)\dag},
\end{equation}
where $\dag$ represents the Moore-Penrose pseudo-inverse.

\paragraph{\textbf{Subproblem of updating $G_i^{(v)} (i<m)$}}
With $F_i^{(v)}$, $W^{(v)}$, $M^{(v)}$, $G^*$, $\beta^{(v)}$ fixed, the optimization problem in Eq. \eqref{eq:obj_func} can be written as follows,
\begin{equation}
    \min_{G_i^{(v)}} \mathcal{C} = \min_{G_i^{(v)}} \Vert X^{(v)} -  ZF_i^{(v)}G_i^{(v)} \Vert^2_W.
\end{equation}
Following \cite{zhao2017multi}, the update rule for $G_i^{(v)} (i<m)$ is defined as,
\begin{equation}\label{eq:update_gi}
    G_i^{(v)} \leftarrow G_i^{(v)} \circ \sqrt{\frac{[Z^TW^{(v)}X^{(v)}]^+ + [Z^TW^{(v)}ZG_i^{(v)}]^-}{[Z^TW^{(v)}X^{(v)}]^- + [Z^TW^{(v)}ZG_i^{(v)}]^+}},
\end{equation}
where $[A]^+ = (|A| + A)/2$ and $[A]^- = (|A| - A)/2$ are element-wise operations.

\paragraph{\textbf{Subproblem of updating $G_m^{(v)}$}}
With $F_i^{(v)}$, $W^{(v)}$, $M^{(v)}$, $G_i^{(v)} (i<m)$, $G^*$, $\beta^{(v)}$ fixed, the optimization problem in Eq. \eqref{eq:obj_func} is defined as,
\begin{equation}
    \min_{G_m^{(v)}} \Vert X^{(v)} -  ZF_m^{(v)}G_m^{(v)} \Vert^2_W -\lambda \beta^{(v)}\Tr(G^*AG^{(v)T}_mM^{(v)}).
\end{equation}
The update formula of $G_m^{(v)}$ is written as follows,
\begin{equation}\label{eq:update_gm}
\begin{gathered}
    G_m^{(v)} \leftarrow G_m^{(v)} \circ \sqrt{\mathcal{U}_n/\mathcal{U}_d}, \\
    \scalebox{0.9}{$
        \displaystyle
        \mathcal{U}_n = [Z^TW^{(v)}X^{(v)}]^++[Z^TW^{(v)}ZG_i^{(v)}]^-+\lambda\beta^{(v)}[M^{(v)}G^*A]^+$}, \\
    \scalebox{0.9}{$
        \displaystyle
        \mathcal{U}_d = [Z^TW^{(v)}X^{(v)}]^-+[Z^TW^{(v)}ZG_i^{(v)}]^++\lambda\beta^{(v)}[M^{(v)}G^*A]^-$}. 
\end{gathered}
\end{equation}

\begin{theorem}\label{th:convergence}
The solution of the update rule in Eq. \eqref{eq:update_gm} satisfies the KKT conditions \cite{kuhn2013nonlinear} and holds convergence property.
\end{theorem}

\begin{proof}
    We define the Lagrangian function as follows,
    \begin{align}
        \mathcal{L}(G_m^{(v)}) & = \sum_v^V  \Vert W^{(v)}(X^{(v)} - F_1^{(v)}F_2^{(v)}\cdots F_m^{(v)}G_m^{(v)}) \Vert^2_F \nonumber \\
    & - \lambda \Tr (G^*A\beta^{(v)}G_m^{(v)T}M^{(v)}) -\eta G_m^{(v)},
    \end{align}
    where $\eta$ is a Lagrange multiplier. The complementary slackness condition gives,
\begin{equation}\label{eq:comp_slack}
    \frac{\partial \mathcal{L}(G_m^{(v)})}{\partial G_m^{(v)}} = (2Z^TW(ZG_m^{(v)} - X^{(v)}) - \lambda \beta^{(v)} M^{(v)} G^*A)G_m^{(v)} = \eta G_m^{(v)} = 0.
\end{equation}
This is a fixed point equation that the solution must satisfy at convergence. Moreover, the solution of Eq. \eqref{eq:update_gm} satisfies the fixed point equation. Let $\overline{G_m^{(v)}}$ be the alternatively updated $G_m^{(v)}$ at any iteration $t$. At convergence, $\overline{G_m^{(v)}} = G_m^{(v)}$, that is,
\begin{equation}\label{eq:new_gm_update}
    G_m^{(v)} \leftarrow G_m^{(v)} \circ \sqrt{\mathcal{U}_n/\mathcal{U}_d},
\end{equation}
where $\mathcal{U}_n$ and $\mathcal{U}_d$ are defined in Eq. \eqref{eq:update_gm}. Using $[A]^+ = (|A| + A)/2$ and $[A]^- = (|A| - A)/2$, Eq. \eqref{eq:new_gm_update} reduces to,
\begin{equation}\label{eq:comp_slack_2}
    (2Z^TW(ZG_m^{(v)} - X^{(v)}) - \lambda \beta^{(v)} M^{(v)} G^*A)(G_m^{(v)})^2 = 0.
\end{equation}
Note that both Eq. \eqref{eq:comp_slack_2} and Eq. \eqref{eq:comp_slack} are identical and share the same factor. Additionally, if $G_m^{(v)}=0$ then $(G_m^{(v)})^2=0$ as well. Therefore if Eq. \eqref{eq:comp_slack} holds, then Eq. \eqref{eq:comp_slack_2} holds as well and inversely.
\end{proof}

\paragraph{\textbf{Subproblem of updating $W^{(v)}$}}
Optimizing Eq. \eqref{eq:obj_func} with respect to $W^{(v)}$ and its constraint is equivalent to optimizing,

\begin{align}
    \mathcal{C} & = \sum_i W_i^{(v)}u_i - \theta(\sum_i {(W_i^{(v)})}^p - 1), \nonumber \\
    & s.t. \hspace{0.1cm} u_i = \sum_j (X^{(v)}-ZF_m^{(v)}G_m^{(v)})^2_{ij}.
\end{align}
Setting $\frac{\partial \mathcal{C}}{\partial W_i} = 0$, and using the KKT complementary slackness condition, we get the following updating formula,
\begin{equation}\label{eq:update_W}
    W_i^{(v)} = \left [\frac{1}{\sum u_i^{\frac{p}{p - 1}}} \right ]^{\frac{1}{p}} u_i^{\frac{1}{p -1}}.
\end{equation}

\paragraph{\textbf{Subproblem of updating $M^{(v)}$}}
With $F_i^{(v)}$, $G_i^{(v)}$, $W^{(v)}$, $G^*$, $\beta^{(v)}$ fixed, the optimization Eq. \eqref{eq:obj_func} can be written as follows,
\begin{equation}\label{eq:opt_align}
    \min_{G^*} - \Tr(M^{(v)}U), s.t. M^{(v)}M^{(v)T} = I_k,
\end{equation}
where $U = \beta^{(v)}G_m^{(v)}A^TG^{*T}$. The problem in Eq. \eqref{eq:opt_align} could also be solved by taking the singular value decomposition of $U$. Moreover, according to Theorem \ref{th:cf-solution}, this optimization problem has a closed-form solution.
\paragraph{\textbf{Subproblem of updating $\beta^{(v)}$}}
With $F_i^{(v)}$, $G_i^{(v)}$, $W^{(v)}$, $G^*$, $M^{(v)}$ fixed, the optimization Eq. \eqref{eq:obj_func} can be written as follows,
\begin{equation}\label{eq:opt_beta}
    \max_{\beta^{(v)}} \beta^{(v)}\omega \hspace{0.3cm} s.t.\Vert \beta^{(v)} \Vert_2 = 1, s.t. \hspace{0.1cm} \beta^{(v)} \geq 0,
\end{equation}
where $\omega = \Tr(G_m^{(v)T}M^{(v)}G^*A)$. This problem could be solved with a closed-form solution as follows,
\begin{equation}
    \beta^{(v)} = \omega/\sqrt{\sum \omega^2}.
\end{equation}

\begin{algorithm}[H]
    \caption{Deep Matrix Factorization with Adaptive Weights for Multi-View Clustering (DMFAW)}
    \label{alg:algorithm}
    \textbf{Input}: $\{X^{(v)}\}_{v=1}^V$: set of multi-view data matrices\\
    $\lambda$: balancing parameter for local and consensus losses\\
    $Tol$: Initial value of Tolerance\\
    Initialize $F_i^{(v)}$, $G_i^{(v)}$, $M^{(v)}$, $\beta^{(v)}$
    \begin{algorithmic}[1] 
        \WHILE{not converged}
        \STATE update $G^*$ by solving Eq. \eqref{eq:opt_cons}
        \FOR{ $v \leq V$}
            \STATE update $W^{(v)}$ using Eq. \eqref{eq:update_W} \\
            \FOR{ $i \leq m$}
                \STATE update $F_i{(v)}$ using Eq. \eqref{eq:update_F} \\
                \STATE update $G_i{(v)}$ using Eq. \eqref{eq:update_gi}
            \ENDFOR
            \STATE update $G_m^{(v)}$ using Eq. \eqref{eq:update_gm} \\
            \STATE update $M^{(v)}$ by solving Eq. \eqref{eq:opt_align} \\
            \STATE update $\beta^{(v)}$ using Eq. \eqref{eq:opt_beta} \\
            \STATE update $p$ using Eq. \eqref{eq:update_p}
        \ENDFOR
        \ENDWHILE
        \STATE \textbf{return} Consensus partition matrix $G^*$ to which we apply K-means to obtain clustering assignment results.
    \end{algorithmic}
\end{algorithm}

\subsection{Discussion}
\paragraph{\textbf{Weight Parameter}} It is important to note that in Eq. \eqref{eq:update_W}, by dynamically adjusting $p$, we can control the degree of feature selection. A smaller $p$ leads to stronger feature selection (highlighting important features), while a larger $p$ results in weaker feature selection (treating all features more equally). This adaptability, provided by Eq. \eqref{eq:update_p} is crucial because different datasets may require different levels of feature selection. For example, in some cases, emphasizing only the most critical features can lead to better clustering, while in others, a more balanced consideration of all features might be preferable.
\paragraph{\textbf{Computational Complexity}}
The proposed algorithm is composed of two stages, which are analyzed separately. To simplify the analysis, we assume that all the layers have the same dimensions $l$. All the data views have the same features $d$, $t$ the number of iterations for both stages, $V$ the number of views and $m$ the number of layers. The complexity of pre-training and fine-tuning stages is $\mathcal{O}(Vmt(dnl + nl^2 + ld^2 + ln^2 + dn^2 + n^2))$ and $\mathcal{O}(Vmt(dnl + nl^2 + dl^2 + nk^2 +  kn^2))$ respectively. Since $l \leq d$ and $k < n$, the time complexity of DMFAW is $\mathcal{O}(Vmt(dnl + ld^2 + dn^2)) + \mathcal{O}(Vmt(dnl + nl^2 + dl^2 + kn^2))$.

\section{Experiments}
\subsection{Experimental setup}
\paragraph{\textbf{Datasets}} We used six benchmark multi-view datasets to assess the performance of our proposed method, \textit{i.e.,} Caltech101-all and Caltech101-7 \cite{fei2007learning}, BBC, BBCSport \cite{greene2006practical}, Handwritten\cite{misc_multiple_features_72}, ORL\cite{samaria1994parameterisation}. The details about these datasets are listed in Table \ref{tab:dataset-details}.
\paragraph{\textbf{Compared methods}} DMFAW is compared with two co-training methods \textbf{Co-reg} \cite{kumar2011cor} and \textbf{Co-train} \cite{kumar2011cot}, and seven matrix decomposition models \textbf{MultiNMF} \cite{liu2013multi}, \textbf{DMVC} \cite{zhao2017multi}, \textbf{MVCF} \cite{zhan2018adaptive}, \textbf{ScaMVC} \cite{huang2018self}, \textbf{AwDMVC} \cite{huang2020auto}, \textbf{MVC-DMF-PA} \cite{zhang2021multi} and \textbf{MCDS} \cite{wang2023multi}.
\paragraph{\textbf{Metrics}} Since ground truth is available for the chosen datasets, we assess the effectiveness of our approach using widely adopted external measures, namely the Purity score, Normalized Mutual Information (NMI) and clustering Accuracy (ACC). These metrics are commonly used for cluster validity evaluation, where higher values signify superior clustering performance.
\paragraph{\textbf{Implementation details}} In our implementation, we initialize the contribution of all local partitions to the consensus partition generation by setting $\beta^{(v)}=1/\sqrt{V}$. The alignment matrix is initially set as $W^{(v)}= I_k$. $Tol$ is initialized to $10^{-3}$. Furthermore, we normalize the multi-view data in all experiments. It is assumed that the true number of clusters $k$ is known and matches the actual number of classes in the datasets. Inspired by the approach in \cite{zhang2021multi}, we adopt a three-layer architecture for all experiments, where the number of components is determined by $[k_1, k_2, k]$, with $k_1$ and $k_2$ chosen from $[8k, 10k, 12k]$ and $[4k, 5k, 6k]$ respectively. To enhance robustness, each experiment is repeated 50 times, mitigating the impact of random initialization in K-means, and the best result is reported.

\begin{table*}[t]
\centering
\resizebox{\textwidth}{!}{%
\begin{tabular}{@{}lcccccccclc@{}}
\toprule
Datasets     & Co-reg & Co-train & MultiNMF & MVCF  & DMVC  & ScaMVC & AwDMVC & MVC-DMF-PA & MCDS  & Ours  \\ \midrule
\multicolumn{11}{c}{ACC(\%)}                                                                               \\ \midrule
Handwritten  & 82.04  & 80.15    & 78.54    & 10.05 & 38.70 & 75.20  & 28.75  & 86.90      & 89.85 & \textbf{90.10} \\
Caltech101-7 & 11.06  & 4.00     & 35.73    & 38.25 & 31.03 & 34.12  & 41.16  & 42.39      & 50.43 & \textbf{51.45} \\
Caltech101-all & 26.37  & 19.88    & 25.76    & 11.75 & 14.89 & 11.60  & 23.86  & 31.73      & 50.40 & \textbf{53.32} \\
BBCSport     & 29.62  & 39.18    & 57.51    & 63.24 & 43.81 & 43.67  & 70.76  & 89.75      & 92.51 & \textbf{96.70} \\
BBC          & 40.61  & 32.71    & 48.26    & 65.75 & 49.48 & 51.95  & 65.04  & 76.16      & 79.23 & \textbf{82.18 }\\
ORL          & 83.25  & 72.50    & 23.75    & 66.50 & 77.00 & 61.75  & 12.00  & 86.75      & 87.20 & \textbf{87.23} \\ \midrule
\multicolumn{11}{c}{Purity(\%)}                                                                            \\ \midrule
Handwritten  & 82.58  & 80.92    & 79.81    & 20.00 & 38.60 & 75.20  & 53.45  & 86.90      & 90.03 & \textbf{90.15} \\
Caltech101-7 & 78.70  & 82.56    & 36.02    & 40.38 & 71.56 & 76.37  & 83.22  & 83.31      & 81.90 & \textbf{87.32} \\
Caltech101-all & 17.30  & 11.15    & 20.17    & 15.40 & 23.67 & 25.20  & 19.52  & 36.31      & 51.49 & \textbf{56.50} \\
BBCSport     & 36.31  & 43.68    & 59.23    & 63.42 & 51.36 & 44.26  & 65.99  & 89.75      & 92.51 & \textbf{96.70} \\
BBC          & 34.24  & 33.15    & 48.25    & 65.84 & 48.38 & 52.56  & 77.55  & 76.16      & 79.23 & \textbf{82.18} \\
ORL          & 85.00  & 76.68    & 23.75    & 68.50 & 79.75 & 66.00  & 12.00  & 87.75      & 88.24 & \textbf{88.25} \\ \midrule
\multicolumn{11}{c}{NMI(\%)}                                                                               \\ \midrule
Handwritten  & 76.26  & 76.59    & 74.64    & 0.45  & 38.65 & 75.64  & 62.93  & 76.58      & 79.45 & \textbf{80.66} \\
Caltech101-7 & 43.33  & 47.30    & 40.01    & 22.84 & 32.05 & 38.54  & 40.25  & 40.97      & \textbf{54.69} & 49.50 \\
Caltech101-all & 33.12  & 39.60    & 41.05    & 23.04 & 25.06 & 35.40  & 37.10  & 38.96      & 48.55 & \textbf{48.67} \\
BBCSport     & 13.18  & 16.48    & 37.96    & 40.45 & 26.04 & 20.36  & 46.82  & 78.80      & 84.60  & \textbf{89.46} \\
BBC          & 11.28  & 10.94    & 27.37    & 42.80 & 20.16 & 20.18  & 45.74  & 51.97      & 57.78 & \textbf{63.70} \\
ORL          & 91.06  & 86.61    & 37.98    & 81.02 & 88.00 & 78.92  & 43.43  & 90.74      & 91.30  & \textbf{91.87} \\ \bottomrule
\end{tabular}%
}
\caption{Accuracy, Purity and NMI comparison of different clustering algorithms on six benchmark data sets. The best results are in bold.}
\label{tab:baseline-comparison}
\end{table*}

\begin{table}[t]
\centering
\resizebox{0.5\columnwidth}{!}{%
\begin{tabular}{@{}lrrr@{}}
\toprule
Dataset      & \#Views & \#Samples & \#Clusters \\ \midrule
Handwritten  & 2       & 2000      & 10         \\
Caltech101-7 & 6       & 1474      & 7          \\
Caltech101-all & 6     & 9144      & 102        \\
BBCSport     & 2       & 544       & 5          \\
BBC          & 4       & 685       & 5          \\
ORL          & 3       & 400       & 40         \\ \bottomrule
\end{tabular}%
}
\caption{Datasets used in our experiments}
\label{tab:dataset-details}
\end{table}

\subsection{Clustering results}
The clustering performance of DMFAW is compared with baseline methods, and the results are presented in Table \ref{tab:baseline-comparison}, where the best-performing results are highlighted in bold. Notably, our proposed method consistently outperforms the baselines across all the six datasets, validating the effectiveness of DMFAW. Particularly noteworthy is its substantial improvement on the BBCSport and BBC datasets compared to existing methods. The improvements for the BBCSport dataset are $4.19\%$ in purity and ACC, and $4.86\%$ in NMI, surpassing the second-best results. Similarly, for the BBC dataset, the gains are $2.95\%$ in purity and ACC, and $5.92\%$ in NMI when compared to the second-best results.\newline
Moreover, when compared with other methods utilizing the deep semi-NMF framework, namely MCDS, DMVC, AwD-MVC, and MVC-DMF-PA, our approach consistently achieves superior results. The use of a weight matrix for effective feature selection, in conjunction with a dynamically updated weight parameter, enables our method to distinguish critical features while adjusting the degree of feature selection based on model performance. This validates the robustness and effectiveness of our proposed DMFAW model in capturing the important features for improved clustering performance.\newline
In summary, the presented quantitative results confirm the effectiveness of our proposed DMFAW in comparison to other state-of-the-art methodologies. Notably, using a dynamic feature selection strategy, via a weight feature matrix improves consensus assignment results.
\begin{figure}[t]
    \centering
    \includegraphics[width=\columnwidth]{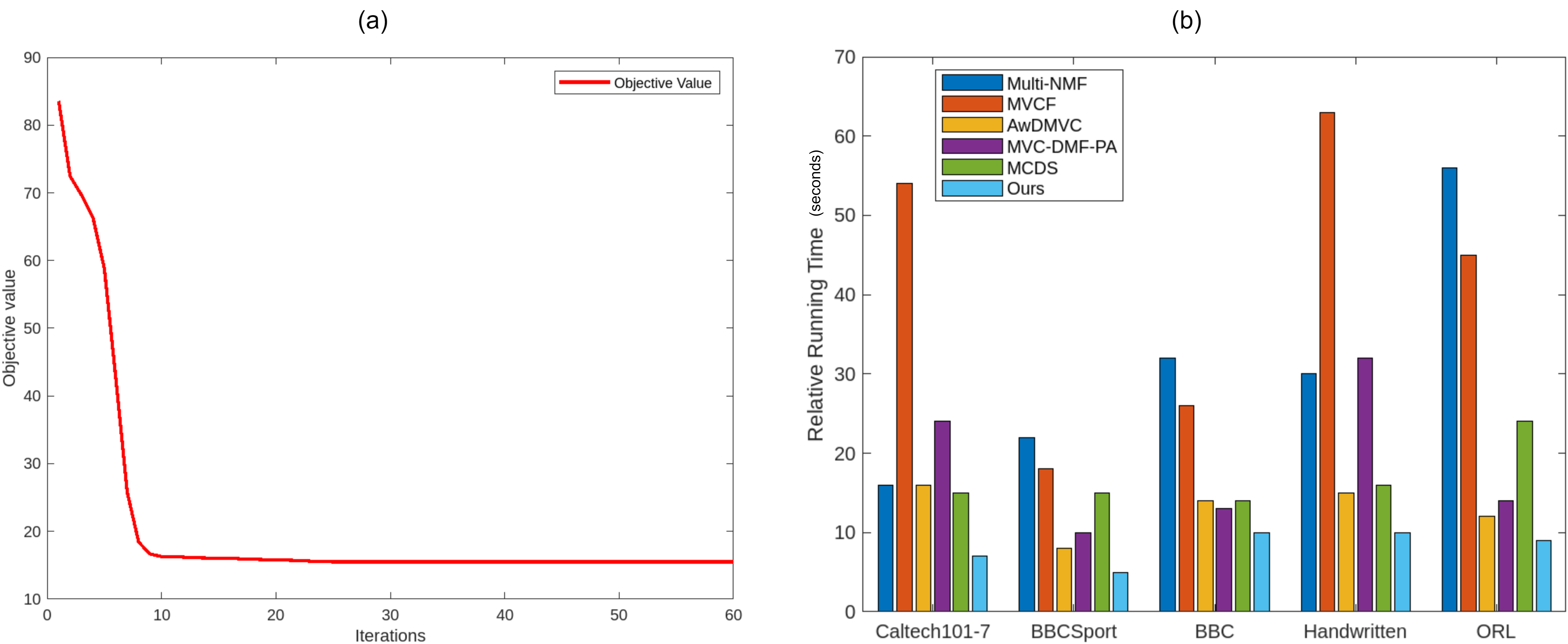}
    \caption{(a) Evolution of the objective value across iterations for Caltech101-7 Dataset. (b) Runtime in seconds, comparing our method to other baseline methods.}
    \label{fig:runtime_convergence}
\end{figure}

\begin{figure*}[]
    \centering
    \includegraphics[width=\textwidth]{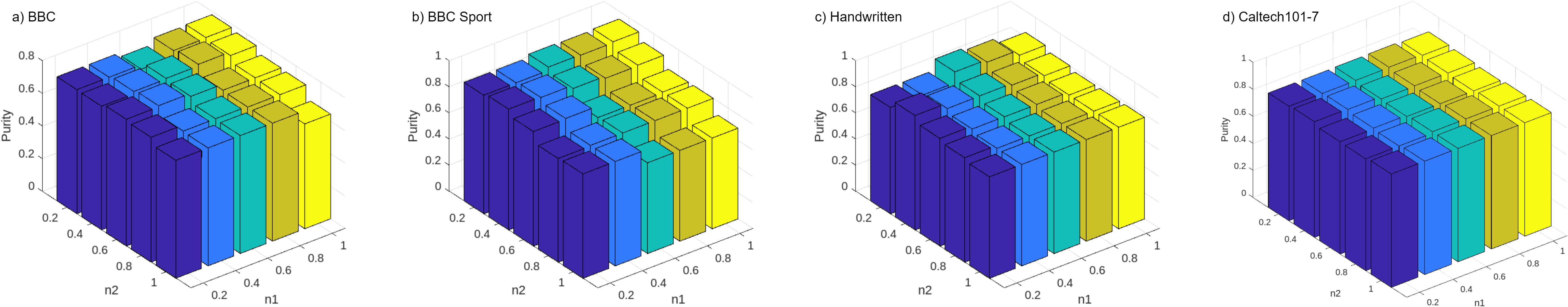}
    \caption{Sensitivity and clustering performance with different parameter settings on four datasets.}
    \label{fig:param_sensitivity}
\end{figure*}

\subsection{Convergence and parameter sensitivity analysis}
\paragraph{\textbf{Convergence Analysis}} We theoretically showed in Theorem \ref{th:convergence} that the updating of $G_m^{(v)}$ satisfies KKT conditions. To experimentally validate the convergence of the entire model, we conducted experiments using the Caltech101-7 dataset, setting hyperparameters to $\lambda = 16$. The evolution of the objective value across iterations is depicted in Figure \ref{fig:runtime_convergence}-a. Notably, the plot illustrates that DMFAW is monotonically decreasing, demonstrating consistent convergence. Moreover, convergence is achieved in fewer than $10$ iterations, underscoring the efficiency of our proposed method, based on PI stepsize control weight parameter update, in accelerating convergence. This property is further investigated in run time experimentation.
\paragraph{\textbf{Run Time}} Figure \ref{fig:runtime_convergence}-b shows that the proposed algorithm demonstrates superior performance in terms of run time, recorded in seconds, compared to other deep matrix factorization methods. This significant reduction in run time can be attributed to our proposed weighted deep matrix factorization combined with the dynamic update of the feature selection degree.
\paragraph{\textbf{Parameter sensitivity}} We conducted a parameter sensitivity study on multiple datasets by varying $n_1$ and $n_2$ across the values $[0.2, 0.4, 0.6, 0.8, 1]$. We aimed to understand how these parameters impact the purity scores, and report the findings in Figure \ref{fig:param_sensitivity}. Notably, our experiments consistently show that the purity scores remain stable across all combinations of $n_1$ and $n_2$. This indicates that even though we introduced two hyperparameters in Eq. \eqref{eq:update_p}, their influence on clustering results is minimal. Therefore, it is reasonable to treat both $n_1$ and $n_2$ as constants, which can be fixed to specific values across all datasets without significantly affecting the performance. In our case, setting $n_1=1$ and $n_2=0.2$ yields consistent and reliable clustering results.

\begin{table}[t]
\centering

\resizebox{0.7\textwidth}{!}{%
\begin{tabular}{@{}lclcl@{}}
\toprule
Datasets & \multicolumn{2}{c}{Dynamic} & \multicolumn{2}{c}{Fixed} \\ \midrule
         & Purity(\%)   & Run Time(s)  & Purity(\%)  & Run Time(s) \\ \midrule
ORL      & 88.25        & 10           & 74.75       & 135         \\
BBCSport & 96.70        & 4            & 73.34       & 37          \\
BBC      & 82.18        & 9            & 64.01       & 122         \\ \bottomrule
\end{tabular}%
}
\caption{Clustering and Run Time performance comparison between fixed and adaptive feature selection.}
\label{tab:ablation_study}
\end{table}

\subsection{Ablation Study}
The comparison between dynamic and fixed weight parameter, as shown in Table \ref{tab:ablation_study}, highlights the clear advantages of using a dynamic update based on control theory principles in clustering tasks. The dynamic approach significantly improves clustering purity and reduces run time across three datasets. For instance, on the BBCSport dataset, there is a $23.36\%$ improvement in purity and a $33$-second reduction in run time compared to the fixed approach. Similar improvements are observed in the other two datasets. These findings demonstrate that dynamically updating the parameter controlling feature selection degree allows DMFAW to converge faster and adapt more effectively to the data.

\begin{figure}[t]
    \centering
    \includegraphics[width=0.9\columnwidth]{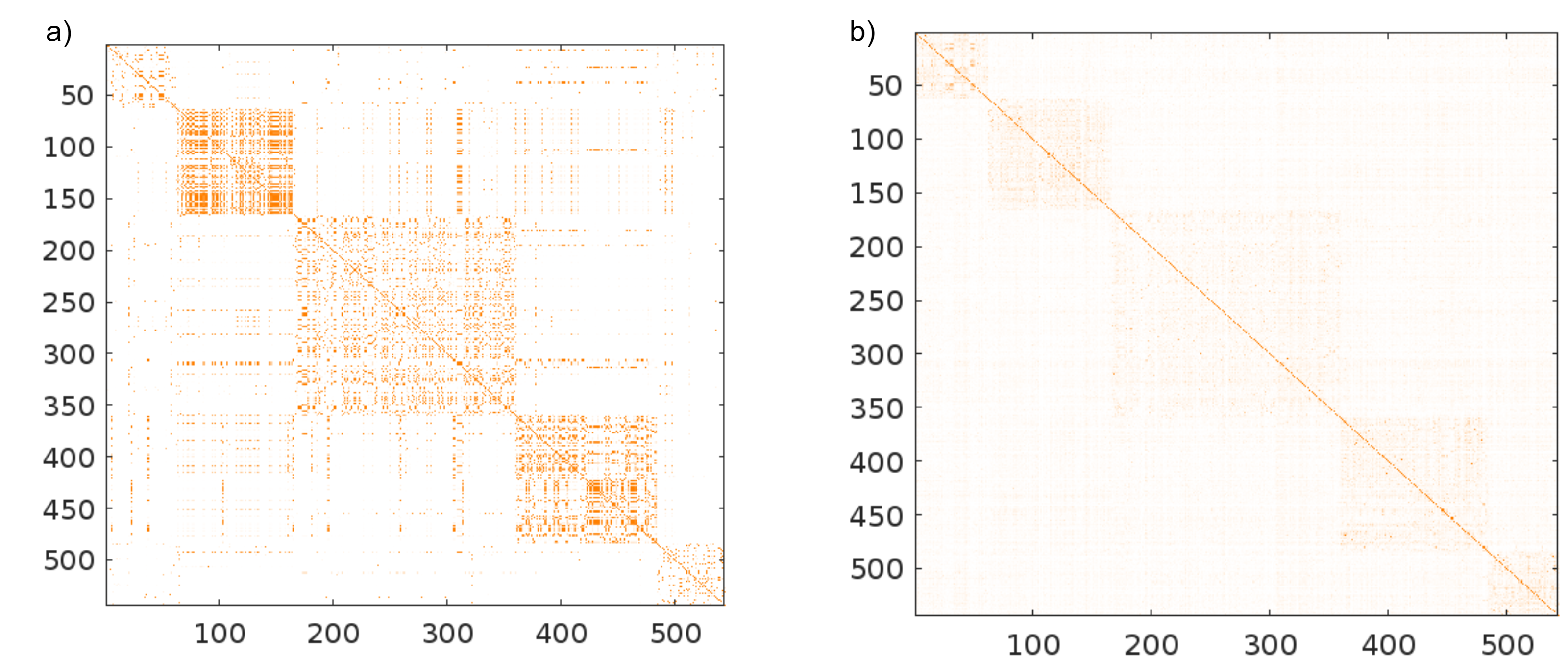}
    \caption{Visualization of pairwise similarity on BBCSport Dataset, using our method (a) and MVC-DMF-PA (b).}
    \label{fig:feature_selection_viz}
\end{figure}

\subsection{Visualization}
To evaluate the effectiveness of the dynamic feature selection with respect to clustering results, we computed the pairwise similarity matrix for the BBCSport dataset, which comprises five clusters. Figure \ref{fig:feature_selection_viz} illustrates the importance of adaptive feature selection by comparing the visual outputs of our approach with those of MVC-DMF-PA.\par
Figure \ref{fig:feature_selection_viz}-a, corresponding to our method, displays distinct, well-defined clusters along the diagonal. This pattern suggests that the adaptive feature selection mechanism effectively emphasizes relevant features while diminishing the influence of irrelevant ones, leading to clearer and more compact clusters. In contrast, Figure \ref{fig:feature_selection_viz}-b, which represents MVC-DMF-PA, shows more diffuse clusters with less distinct boundaries. This lack of clear cluster structure indicates that the features are not as effectively leveraged in MVC-DMF-PA, resulting in reduced clustering quality.

\section{Conclusion}
This paper introduces Deep Matrix Factorization with Adaptive Weights for Multi-View Clustering (DMFAW). Using a weighted Deep Semi-NMF methodology, DMFAW simultaneously extracts local partition matrices and performs feature selection, significantly enhancing the robustness of multi-view clustering. A dynamic parameter update mechanism, inspired by Control Theory's PI Stepsize Control, ensures feature selection adaptability to diverse datasets while accelerating convergence. Extensive experiments on benchmark datasets demonstrate the effectiveness and efficiency of DMFAW, and its superior performance compared to other state-of-the-art methods. Additionally, the success of our approach is influenced by the quality of the views. In the future, we will explore methods to improve the robustness of our approach in the presence of noisy views.

\newpage
\bibliographystyle{elsarticle-num-names} 
\bibliography{refs.bib}





\end{document}